\documentclass{turing2012}
\usepackage{times}
\usepackage{graphicx, color, caption, subcaption}
\usepackage{latexsym}
\usepackage{amsmath, amsfonts, amsthm, mathrsfs}

\newcommand{\e}{\varepsilon}

\begin{document}
\newtheorem{dfn}{Definition}
\newtheorem{lem}[dfn]{Lemma}
\newtheorem{prop}[dfn]{Proposition}
\newtheorem{thm}[dfn]{Theorem}
\newtheorem{cor}[dfn]{Corollary}
\newtheorem{fact}[dfn]{Fact}
\newtheorem {exa}[dfn]{Example}
\newcommand{\vsp}{\vspace{6mm}}
\newcommand{\baralp}[2]{\mbox{$\bar{\alpha}_{#1}^{( #2 )}$}}
\newcommand{\alp}[2]{\mbox{$\alpha_{#1}^{( #2 )}$}}
\newcommand{\vex}[1]{\vec{ #1}}
\newcommand{\V}[1]{\mbox{$V^{(#1)}$}}
\newcommand{\VK}[2]{\mbox{$V^{(#1)}(#2)$}}
\newcommand{\SL}[1]{\mbox{$SL^{(#1)}$}}
\newcommand{\LA}[1]{\mbox{$L^{(#1)}$}}
\newcommand{\fn}{\mbox{$f^{(n)}$}}
\newcommand{\pr}[1]{\mbox{$Prob^{(#1)}$}}

\title{Emerging Dimension Weights in a Conceptual Spaces Model of Concept Combination}

\author{Martha Lewis \and Jonathan Lawry \institute{University of Bristol, England, email: martha.lewis@bristol.ac.uk, j.lawry@bristol.ac.uk} }

\maketitle
\bibliographystyle{AISB}

\begin{abstract}
We investigate the generation of new concepts from combinations of properties as an artificial language develops. To do so, we have developed a new framework for conjunctive concept combination. This framework gives a semantic grounding to the weighted sum approach to concept combination seen in the literature. We implement the framework in a multi-agent simulation of language evolution and show that shared combination weights emerge. The expected value and the variance of these weights across agents may be predicted from the distribution of elements in the conceptual space, as determined by the underlying environment, together with the rate at which agents adopt others' concepts. When this rate is smaller, the agents are able to converge to weights with lower variance. However, the time taken to converge to a steady state distribution of weights is longer.
\end{abstract}

\section{INTRODUCTION}
Humans are skilled at making sense of novel combinations of concepts, so to create artificial languages for implementation in AI systems, we must model this ability. Standard approaches to combining concepts, e.g. fuzzy set theory, have been shown to be inadequate \cite{osh1981}. Concepts formed through the combination of properties frequently have `emergent attributes' \cite{hamp1987} which cannot be explicated by decomposing the label into its constituent parts.  We have developed a model of concept combination within the label semantics framework as given in \cite{lawry2009, lewishierarchical}. The model is inspired by and reflects results in \cite{hamp1987}, in which membership in a compound concept can be rendered as the weighted sum of memberships in individual concepts, however, it can also account for emergent attributes, where e.g. importance in a conjunctive concept is greater than the importance of an attribute in the constituent concepts. We implement a simple version of this model in a multi-agent simulation of language users, and show that the agents converge to shared combination weights, allowing effective communication. These weights are determined by the distribution of objects in the agents' conceptual space. This provides a theoretical grounding to the proposal seen in the literature \cite{gard2004,hamp1987,lakoffhedges,zadehhedges} that complex concepts can be characterised as weighted sums of attributes. Further, it relates the weights in the combined concept to the external world. In this paper, we firstly summarise the theoretical framework we use (section \ref{sec:bkg}), and in section \ref{sec:cmp} give a brief account of our model of concept combination. In section \ref{sec:expts}, we implement a simple version of our model in a multi-agent simulation of a community of language users in order to examine whether an how such a community is able to converge to shared combination weights. We give simulation methods and results, and analyse the behaviour of the agents. Finally, section \ref{sec:disc} discusses our results and gives an indication of future work.

\section{BACKGROUND}
\label{sec:bkg}
We model concepts within the label semantics framework \cite{lawry2004, lawry2009}, combined with prototype theory \cite{rosch} and the conceptual spaces model of concepts \cite{gard2004}. Prototype theory offers an alternative to the classical theory of concepts, basing categorization on proximity to a prototype. This approach is based on experimental results where human subjects were found to view membership in a concept as a matter of degree, with some objects having higher membership than others \cite{rosch}. Fuzzy set theory \cite{zadeh1965}, in which an object $x$ has a graded membership $\mu_L(x)$ in a concept $L$, was proposed as a formalism for prototype theory. However, numerous objections to its suitability have been made \cite{ hamp, hamp1987, kp, osh1981, smith}. 

Conceptual spaces theory renders concepts as convex regions of a \emph{conceptual space} - a geometrical structure with quality dimensions and a distance metric. Examples are: the RGB colour cube, pictured in figure \ref{fig:rgbcube}; physical dimensions of height, breadth and depth; or the taste tetrahedron. Since concepts are convex regions of such spaces, the centroid of such a region can naturally be viewed as the prototype of the concept.

\begin{figure}[h]
\centering
\includegraphics[width = 0.3\textwidth]{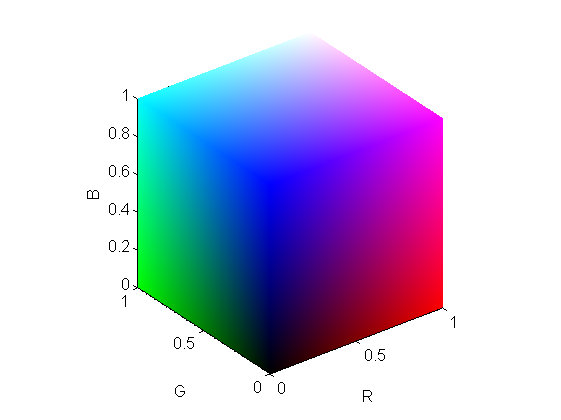}
\caption{The RGB cube represents colours in three dimensions of Red, Green and Blue. A colour concept such as `purple' can be represented as a region of this conceptual space.}
\label{fig:rgbcube}
\end{figure}

Label semantics proposes that agents use a set of labels $LA = \{L_1, ..., L_n\}$ to describe a conceptual space $\Omega$ with distance metric $d(x, x')$. Labels $L_i$ are associated with prototypes $P_i \subseteq \Omega$ and uncertain thresholds $\e_i$, drawn from probability distributions $\delta_{\e_i}$. The threshold $\e_i$ captures the notion that an element $x \in \Omega$ is sufficiently close to $P_i$  to be labelled $L_i$. The membership of an object $x$ in a concept $L_i$ is quantified by $\mu_{L_i}(x)$, given by 

\[
\mu_{L_i}(x) = P(d(x, P_i) \leq \e_i) = \int_{d(x, P_i)}^\infty  \delta_{\e_i}(\e_i)\mathrm{d}\e_i 
\]

Labels can then be described as $L_i = <\!\!P_i, d(x, x'), \delta_{\e_i}\!\!>$. We illustrate this idea in figure \ref{fig:pnth}.

\begin{figure}[h!]
\centering
  \def\svgwidth{0.5\columnwidth}  
  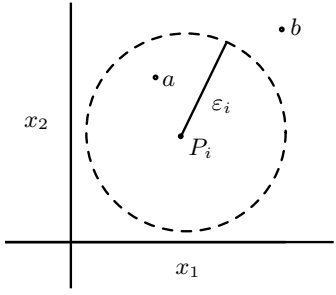  
\caption{Prototype-threshold representation of a concept $L_i$. The conceptual space has dimensions $x_1$ and $x_2$. The concept has prototype $P_i$ and threshold $\e_i$. The uncertainty about the threshold is represented by the dotted line. Element $a$ in the conceptual space is within the threshold, so it is appropriate to assert `$a$ is $L_i$'. Element $b$ is outside the threshold, so it is not appropriate to assert `$b$ is $L_i$'.}
\label{fig:pnth}
\end{figure}

\section{A HIERARCHICAL MODEL OF CONCEPT COMPOSITION}
\label{sec:cmp}
Experiments in the psychological literature  propose that human concept combination can in many cases be modelled as a weighted sum of attributes such as `has feathers', `has a beak' (for the concept `Bird') \cite{hamp1987}. These attributes differ from quality dimensions in conceptual spaces: they tend to be binary, complex, and multidimensional in themselves. We therefore view each attribute as a label in a continuous conceptual space $\Omega_i$ and combine these labels in a binary conceptual space $\{0,1\}^n$ illustrated in figure \ref{fig:binspace}. In this binary conceptual space, a conjunction of such labels ${\alpha} = \bigwedge_{i = 1}^n\pm L_i$ maps to a binary vector $\vec{y}_\alpha$ taking value $1$ for positive labels $L_i$ and $0$ for negated labels $\neg L_i$ (figure \ref{fig:binspace3}).
\begin{figure}[h!]
\small
  \centering  
  \def\svgwidth{0.75\columnwidth}  
  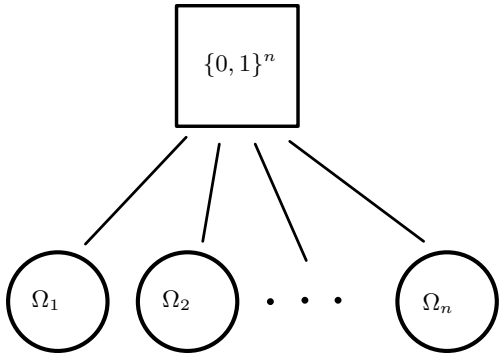  
\caption{Combining labels in a binary space}
\label{fig:binspace}
\end{figure}

\begin{figure}[h!]
\centering
  \def\svgwidth{0.5\columnwidth}  
  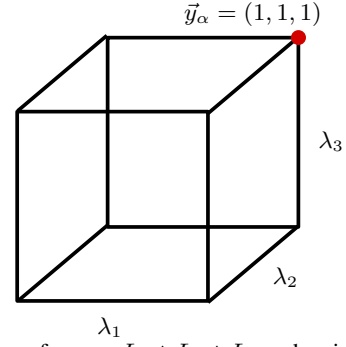  
\caption{Prototype for ${\alpha} = L_1 \land L_2 \land L_3$ and weighted dimensions in binary space.}
\label{fig:binspace3}
\end{figure}

 We treat membership in ${\alpha}$ in the binary conceptual space within the label semantics framework. So ${\alpha}$ is described in the binary space by $\tilde{\alpha}=<\!\!\vec{y}_\alpha, d(\vec{y}, \vec{y}'), \delta\!\!>$ as before, and $\mu_{\alpha}(y) = \int_{d(y, y_\alpha)}^\infty  \delta_{\e}(\e)\mathrm{d}\e$. Since the dimensions of the space are weighted, some are considered more important than others. The presence or absence of some attributes may be relaxed. For example, `Bird' might be characterised by (among others) the attributes `has feathers', `has wings', `flies'. The attributes `has wings' and `has feathers' should be given more importance than `flies'. This is because various species of birds do not fly, so this attribute may be relaxed whilst still allowing something to be categorised as a bird. The effect this has is to create elliptical regions of the conceptual space, as illustrated in figure \ref{fig:binspace4}.

\begin{figure}[h!]
\centering
  \def\svgwidth{0.5\columnwidth}  
  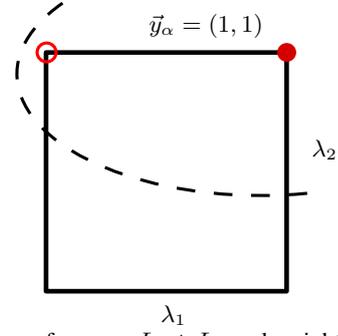  
\caption{Prototype for $\alpha = L_1 \land L_2$ and weighted dimensions in a two dimensional binary space, together with the threshold $\e$ for $\alpha$. The point $(0, 1)$, indicated by an open circle, can be considered to be an instance of the concept for which $y_\alpha$ is the prototype.}
\label{fig:binspace4}
\end{figure}

The distance metric between two vectors $\vec{y},\vec{y}^\prime$ in the binary space $\{0,1\}^n$ is written $H_{\vec{\lambda}}(\vec{y},\vec{y}^\prime)$ and defined as a weighted city block metric. 

Then, under certain constraints, membership in the combined concept $\alpha = \bigwedge_{i = 1}^n\pm L_i$ is equal to the weighted sum of membership in each of the $L_i$. This is stated in the following theorem:

\begin{thm}[Compound Concepts]
\label{thm:sumthm}
Let $L_i$ be attributes described by membership functions $\mu_{L_i}(x_i)$ in conceptual spaces $\Omega_i$. Let $\alpha$ be a conjunction of such attributes (or their negation):  $\alpha=\bigwedge_{i=1}^n \pm L_i$. If we combine these labels in a binary space $\{0,1\}^n$ with $\tilde{\alpha}=<\!\!\vec{y}_\alpha, H_{\vec{\lambda}}, U(0,\lambda_T)\!\!>$ where $\lambda_T = \sum_{i = 1}^n \lambda_i$, then we may calculate the membership in $\alpha$ in the space $\Omega_1 \times ... \times \Omega_n$ by:
\begin{gather*}
\mu_{{\alpha}}(\vec{x})=\sum_{i=1}^n \frac{\lambda_i}{\lambda_T} \mu_{\pm L_i}(x_i)
\end{gather*}
where $\vec{x} \in \Omega_1 \times ... \times \Omega_n$, $x_i \in \Omega_i$.

\end{thm}

We may further combine such compound concepts ${\theta}, {\varphi}$ in a higher level binary space.
Then, again under certain constraints, $ \tilde{\theta}\bullet \tilde{\varphi}$ can be expressed in the continuous space as a weighted sum of ${\theta}$ and ${\varphi}$.
\begin{thm}[Conjunction of Compound Concepts]
\label{conjthm}
 Let $\tilde{\theta} \bullet \tilde{\varphi} = <\!\!\{(1,1)\}, H_{\vec{w}}, \delta\!\!>$, where $\theta$ and $\varphi$ are compound concepts as described in theorem \ref{thm:sumthm}, so that $\theta$ is characterised by membership function $\mu_{{\theta}}(\vec{x})=\sum_{i=1}^n \frac{\lambda_{\theta_i}}{\lambda_{\theta_T}} \mu_{\pm L_i}(x_i)$ for $\vec{x} \in \Omega_1 \times ... \times \Omega_n$, $x_i \in \Omega_i$, and $\varphi$ is similarly characterised. Then 
\[ 
\mu_{\tilde{\theta}\bullet\tilde{\varphi}}(\vec{x}) = \sum_{i=1}^n (\frac{w_1 \lambda_{\varphi_T} \lambda_{\theta_i} + w_2 \lambda_{\theta_T} \lambda_{\varphi_i}}{w_T \lambda_{\theta_T} \lambda_{\varphi_T}}) \mu_{\pm L_i}(\vec{x})
\]
where $\vec{x} \in \Omega_1 \times ... \times \Omega_n$, $x_i \in \Omega_i$.
\end{thm}

These results show that under the constraint $\e \sim U(0,\lambda_T)$, combining labels in a weighted binary conceptual space leads naturally to the creation of compound concepts as weighted sums of individual labels, reflecting results in \cite{hamp1987}. An important aspect of these results is that non-compositional phenomena are seen, such as emergent attributes. These are attributes that become more important in the conjunction of two concepts than in either of the constituent concepts. A specific example seen in \cite{hamp1987} is that the attribute`talks' becomes more important in the conjunction `Birds that are Pets' than in either `Birds' or `Pets'. Relaxing the constraint $\e \sim U(0,\lambda_T)$ allows us to account for phenomena such as emergent attributes and overextension of concepts. In the current paper, however, we concentrate on a simple weighted sum combination and examine the properties of this type of combination in multi-agent simulations.

\section{CONVERGENCE OF DIMENSION WEIGHTS ACROSS A POPULATION}
\label{sec:expts}
We implement a simple version of the hierarchical model of concept combination in a multi-agent simulation of agents playing a series of language games, similar to those used in \cite{belp}. We investigate how agents using compound concepts $\alpha$ in a conceptual space converge to a shared weighting of the constituent concepts $L_i$. Agents do indeed converge to a shared weighting, which is dependent on the distribution of objects in the environment and also on the rate at which they move towards other agents' concepts. Section \ref{sec:methods} describes the methods and simulation set-up. Section \ref{sec:simulations} gives simulation results, and section \ref{sec:analysis} gives an analysis of the results.

\subsection{Methods}
\label{sec:methods} 
Consider agents each with labels $L_1 = <\!\!P_1, d(x,y), \delta_1\!\!> \in \Omega_1$, $L_2=<\!\!P_2, d(x,y), \delta_2\!\!> \in \Omega_2$. We assume that agents combine these two labels as in section \ref{sec:cmp} - i.e., in a binary space $\{0,1\}^2$ with weight vector $\vec{\lambda} = (\lambda_1, \lambda_2)^\top$ and where the threshold in the binary space $\e$ has distribution $\delta = U(0, \lambda_T)$, where $\lambda_T = \lambda_1 + \lambda_2$. Then, membership in the conjunction $L_1 \land L_2$ in the space $\Omega_1 \times \Omega_2$ may be calculated as a weighted sum of membership in the individual spaces. W.l.o.g. we assume that $\lambda_T = 1$. We therefore have  $\mu_{L_1 \land L_2}(\vec{x}) = \lambda \mu_{L_1}(x_1) + (1 - \lambda) \mu_{L_2}(x_2)$.

To investigate how these weights are to be determined, we run simulations in which agents with equal labels but randomly initiated weights engage in a series of dialogues about elements in the conceptual space, adjusting their weights after each dialogue is completed.

\subsubsection{Assertion algorithm}
Agents are equipped with shared labels $L_1 \in \Omega_1 = [0,1]$ and $L_2 \in \Omega_2 = [0,1]$, and a weight $\lambda$. At each timestep agents are paired into speaker and listener agents. Each pair of agents is shown an element $x \in \Omega_1 \times \Omega_2 = [0,1]^2$. The speaker agent asserts one of 
\begin{align*}
\alpha_1 &= L_1 \land L_2\\
\alpha_2 &= L_1 \land \neg L_2\\
\alpha_3 &= \neg L_1 \land L_2\\
\alpha_4 &= \neg L_1 \land \neg L_2\\
\end{align*}
where the membership in the compound concept is determined by the weighted sum of the memberships in the constituent concepts.
\begin{align*}
\mu_{\alpha_1}(x) &= \lambda \mu_{L_1}(x_1) + (1 - \lambda) \mu_{L_2}(x_2)\\
\mu_{\alpha_2}(x) &= \lambda \mu_{L_1}(x_1) + (1 - \lambda)(1 -  \mu_{L_2}(x_2))\\
\mu_{\alpha_3}(x) &= \lambda (1 - \mu_{L_1}(x_1)) + (1 - \lambda) \mu_{L_2}(x_2)\\
\mu_{\alpha_4}(x) &= \lambda (1 - \mu_{L_1}(x_1)) + (1 - \lambda)(1 -  \mu_{L_2}(x_2))\\
\end{align*}
The concept $\alpha_i$ asserted is that for which $\mu_{\alpha_i}(x)$ is maximal. Note that this implies that if $\alpha_i$ asserted then $\mu_{\alpha_i}(x) \geq 0.5$.

\subsubsection{Updating algorithm}
\label{sec:updating}
We compare two different updating algorithms. The first implements the idea that the listener agent updates its concepts when the agent's belief in the appropriateness of a compound label, $\mu_{\alpha_i}(x)$, is less than the reliability of the speaker as measured by a weight $w \in [0, 1]$. So if $\mu_{\alpha_i}(x) \leq w$ the listener agent updates its label set.

The update consists in moving the dimension weight $\lambda$ towards the value $A$ which satisfies $\mu_{\theta}(x) = w$, where  
\[
A = \frac{w - \mu_{L_2}(x_2)}{\mu_{\pm L_1}(x_1) - \mu_{\pm L_2}(x_2)}
\]
However, it is possible for $A < 0$ or $A > 1$, in which cases we set $A = 0$ or $A = 1$, giving us:
\begin{align*}
A = 
\begin{cases}
1 & \quad\text{if }\frac{w - \mu_{L_2}(x_2)}{\mu_{\pm L_1}(x_1) - \mu_{\pm L_2}(x_2)} >1\\
0 & \quad\text{if }\frac{w - \mu_{L_2}(x_2)}{\mu_{\pm L_1}(x_1) - \mu_{\pm L_2}(x_2)} <0\\
\frac{w - \mu_{L_2}(x_2)}{\mu_{\pm L_1}(x_1) - \mu_{\pm L_2}(x_2)} &\quad \text{otherwise}
\end{cases}
\end{align*}
The listener agent updates the weight $\lambda$ to $\lambda'$ by:

\[
\lambda' = \lambda + h(A - \lambda)
\]

We measure the convergence between $\lambda_i$ across the agents as the standard deviation (SD) of the $\lambda_i$. 

We also introduce a second updating algorithm. This is similar to the first with the exception that the criterion for updating is that $\mu_{\alpha_i}(\vec{x}) \neq w$. This means that if an agent with low reliability makes an assertion, the listener agent shifts the dimension weight $\lambda$ to reduce the appropriateness of the assertion $\alpha_i$.

\subsubsection{Simulation Details}
Agents have labels $L_1 = L_2 = <\!\!1, d, U[0,1]\!\!>$, where $d$ is Euclidean distance, to describe the conceptual space $\Omega = \Omega_1 \times \Omega_2 = [0,1]^2$. In this case, we have $\mu_{L_i}(x) = x_i$. Simulations are run with $10$ agents for $2,000$ timesteps with increment $h = 10^{-3}$ unless otherwise indicated.  At each timestep, each agent talks to every other agent, in a randomised order. Simulation results are averaged across 25 runs. 

Parameters varied are the distribution of elements within the conceptual space and the reliability of the agents. So, for example, one simulation might include elements sampled uniformly across the whole conceptual space, whereas another might include elements sampled uniformly from one half of the space. This difference in distribution leads to differences in the combination weights.

\subsection{Simulation Results}
\label{sec:simulations}
\subsubsection{Updating model 1: $\mu_{\alpha_i}(\vec{x}) < w$}
Within this updating model, the listener agent updates only when $\mu_{\alpha_i}(\vec{x}) < w$. We find that when the agent reliability $w > 0.5$, agents are able to converge to a shared weight $\lambda$.  The weight to which agents converge is dependent on the distribution of objects in the environment, and the reliability $w$ of the agents. When $w = 1$ for all agents, and $x_1 \sim U[0,1]$, $x_2 \sim U[0,0.5]$, so that elements are encountered within half of the total space $\Omega = [0,1]^2$, the agents converge to a value of $\lambda = 0.5$, as seen in figure \ref{fig:fig1all}.

ze

When the distribution of elements in the environment is changed, $\lambda$ may converge to a different value. For example, changing the distribution of the elements in the space to $x_1 \sim U[0.25,0.75]$, $x_2 \sim U[0,0.5]$, with $w = 1$ for all agents, results in a final value of $\lambda = 0.25$, illustrated in figure \ref{fig:fig2all}.
\begin{figure} [h!]
        \centering
        \begin{subfigure}[t]{0.65\columnwidth}
                \hspace{5pt}
                \includegraphics[width=\textwidth]{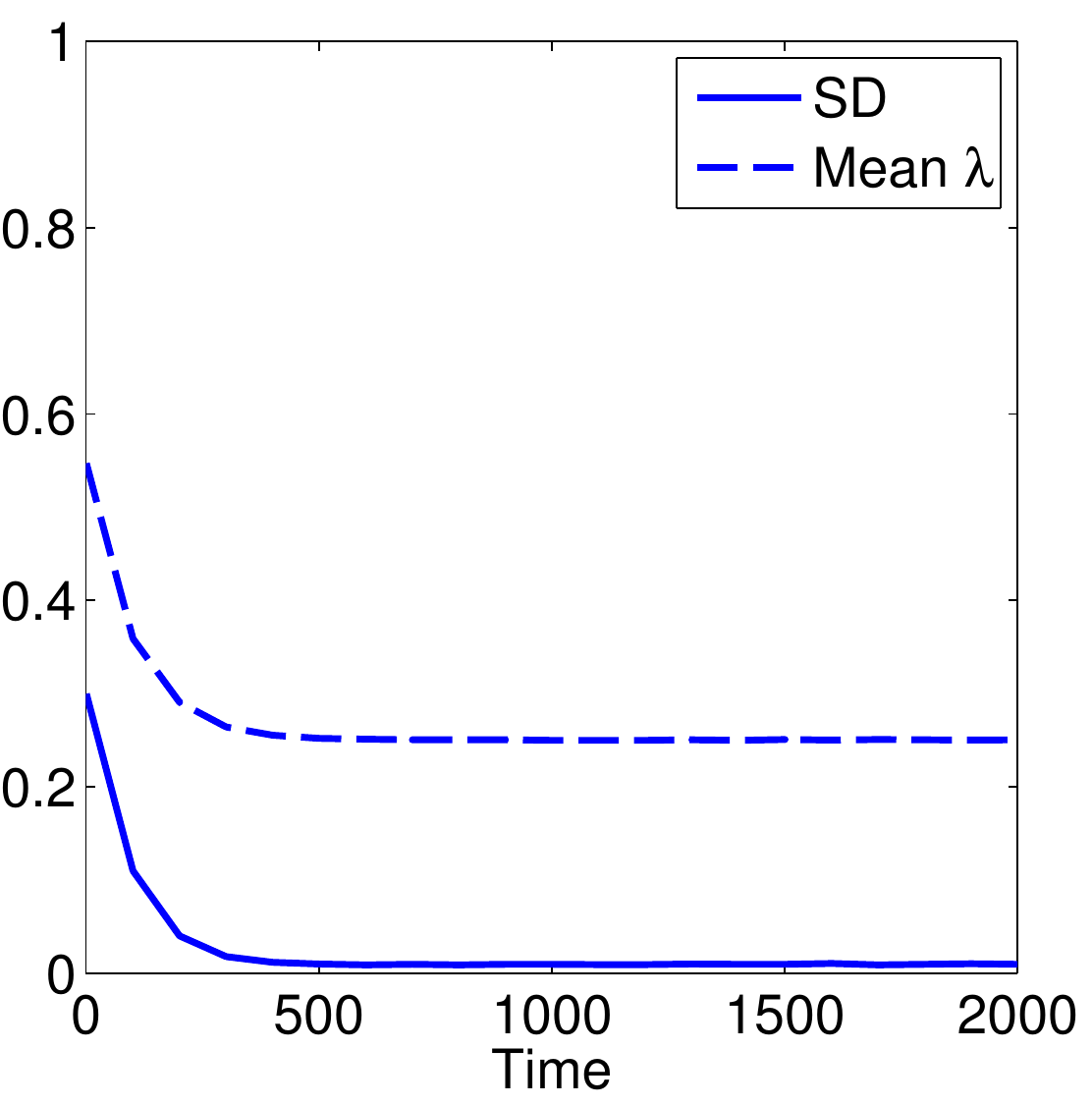}
                \caption{SD and mean $\lambda$ over time.}
                \label{fig:fig2}
        \end{subfigure}
        \begin{subfigure}[t]{0.7\columnwidth}
                \centering
                \includegraphics[width=\textwidth]{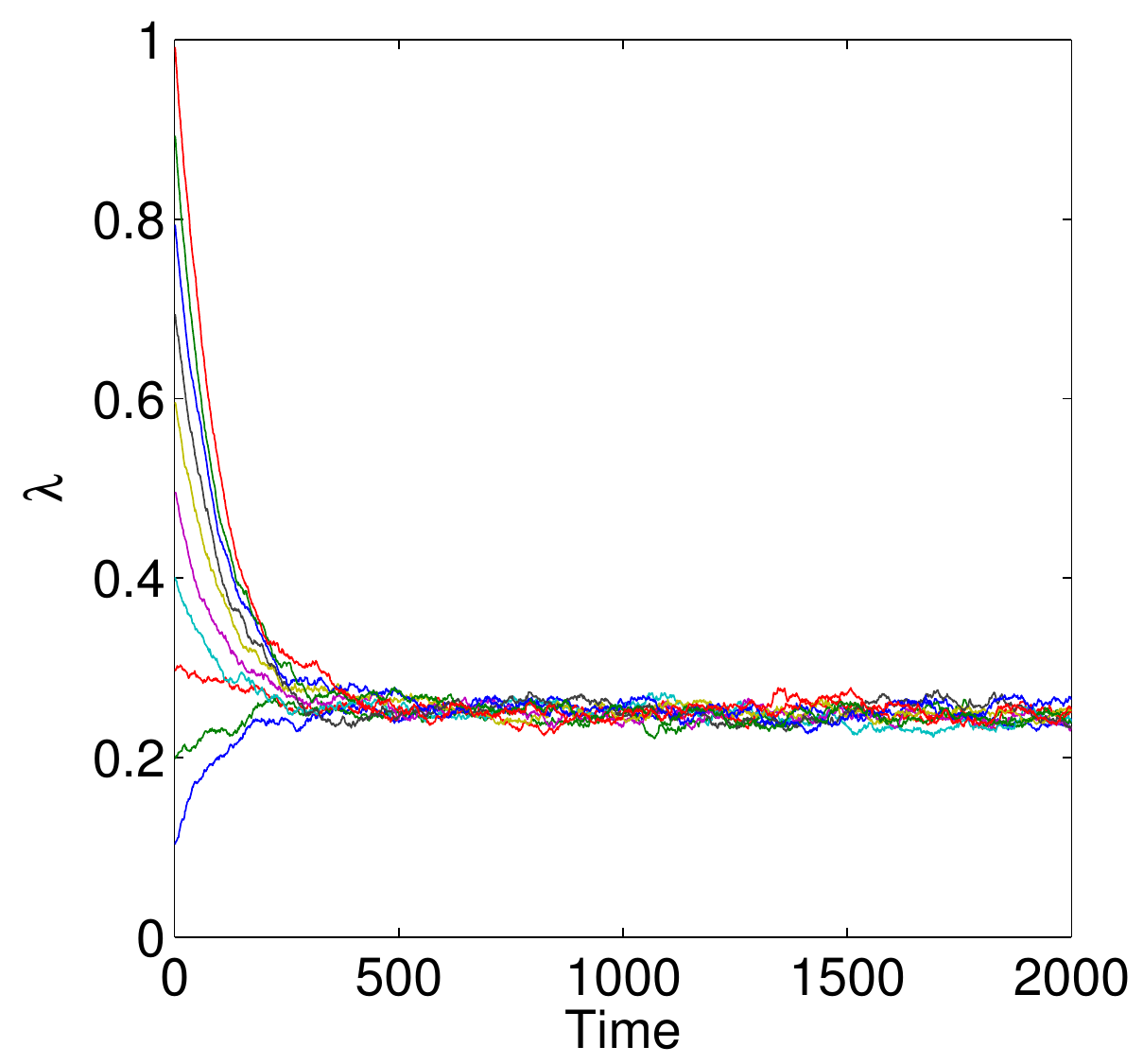}
                \caption{Behaviour of individual $\lambda_i$ over time (results from just one simulation).}
                \label{fig:mlfig2}
        \end{subfigure}
	\caption{$x_1 \sim U[0.25,0.75]$, $x_2 \sim U[0,0.5]$, $w = 1$ for all agents.}
	\label{fig:fig2all}
\end{figure}

The final value of $\lambda$ may also be dependent on the reliability of the agents as parameterised by $w$. For some distributions, the value of $w$ does not affect the value of $\lambda$ to which agents converge. In others, the value of $w$ alters the weighting $\lambda$. This is illustrated in figure \ref{fig:cmp1}.


\begin{figure} [h!]
        \centering
        \begin{subfigure}[t]{0.65\columnwidth}
                \centering
                \includegraphics[width=\textwidth]{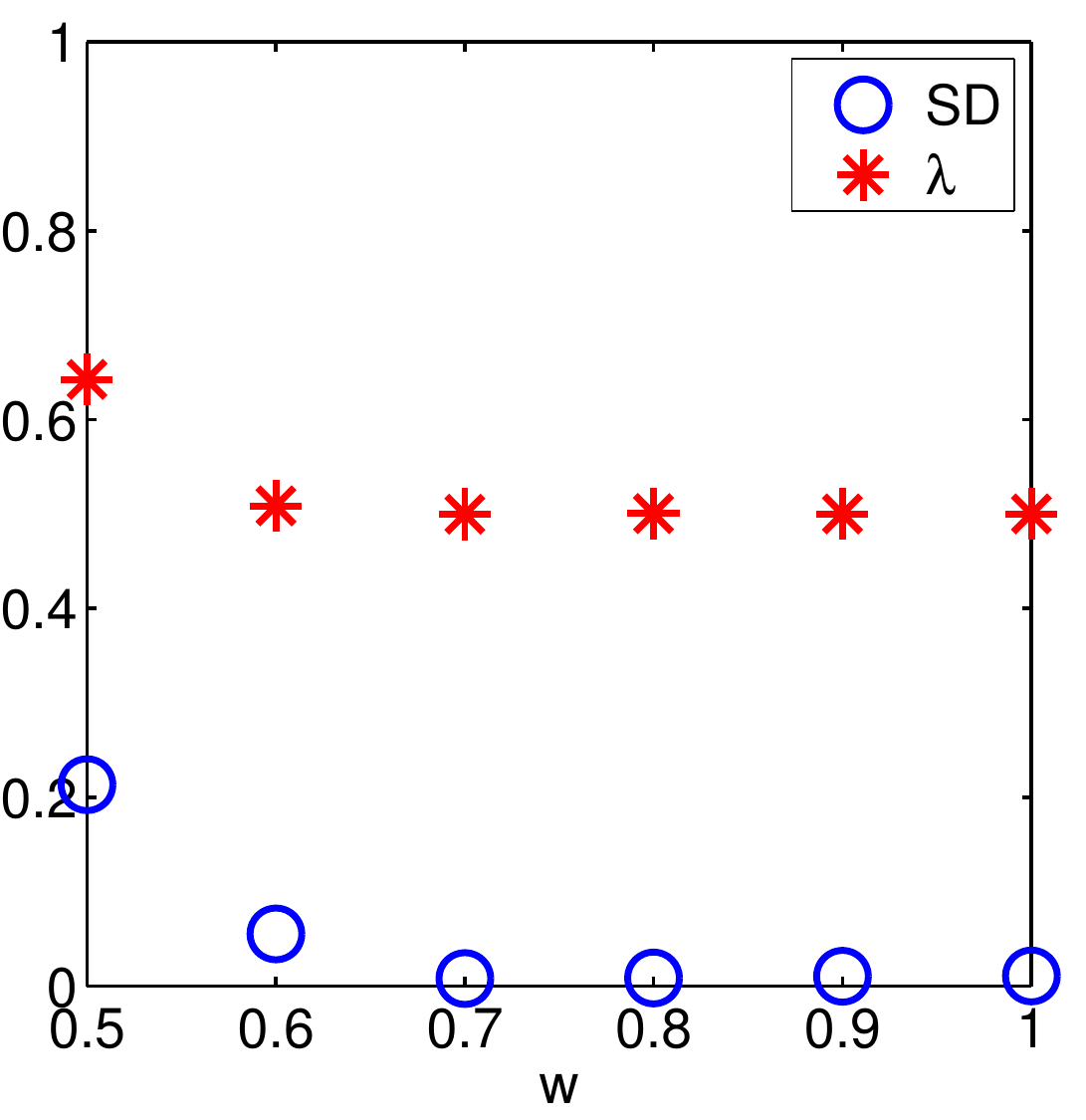}
                \caption{$x_1 \sim U[0,1]$, $x_2 \sim U[0,0.5]$. Values of $w$ vary as indicated.}
                \label{fig:cmpeqAPD1}
        \end{subfigure}
        ~ 
        \begin{subfigure}[t]{0.65\columnwidth}
                \centering
                \includegraphics[width=\textwidth]{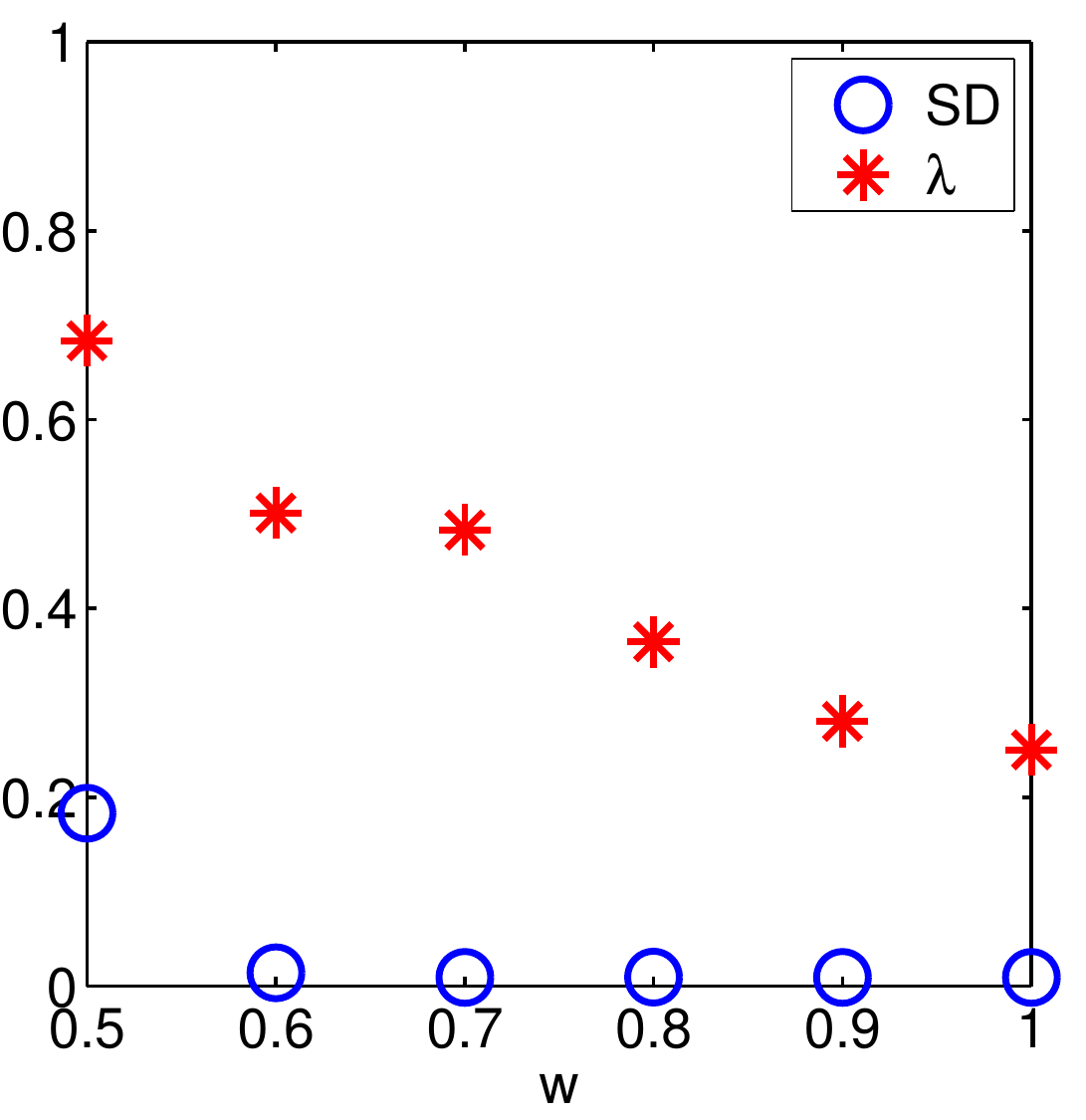}
                \caption{$x_1 \sim U[0.25,0.75]$, $x_2 \sim U[0,0.5]$. Values of $w$ vary as indicated.}
                \label{fig:cmpeqmean1}
        \end{subfigure}
	\caption{SD and mean $\lambda$ at time $t = 2000$ for different values of $w$. For $x_1 \sim U[0,1]$, $x_2 \sim U[0,0.5]$, provided that $w > 0.5$, agents converge to $\lambda = 0.5$ (LHS). In contrast, for $x_1 \sim U[0.25,0.75]$, $x_2 \sim U[0,0.5]$, agents converge to varying values depending on the value of $w$ (RHS).}
	\label{fig:cmp1}
\end{figure}

\subsubsection{Updating model 2: $\mu_{\alpha_i}(\vec{x}) \neq w$}
We also introduce a second updating model, in which the listener agent updates whenever $\mu_{\alpha_i}(\vec{x}) \neq w$. Although this model is slightly less realistic, it is more amenable to analysis, and so we use this as a starting point for the analysis of these systems. Again, each set of simulations is run 25 times and results are averaged. Figure \ref{fig:cmpeq075} compares the results of simulations run using the two different assertion models, where $x_1 \sim U[0.25,0.75]$, $x_2 \sim U[0,0.5]$.

\begin{figure} [h!]
        \centering
        \begin{subfigure}[t]{0.7\columnwidth}
                \centering
                \includegraphics[width=\textwidth]{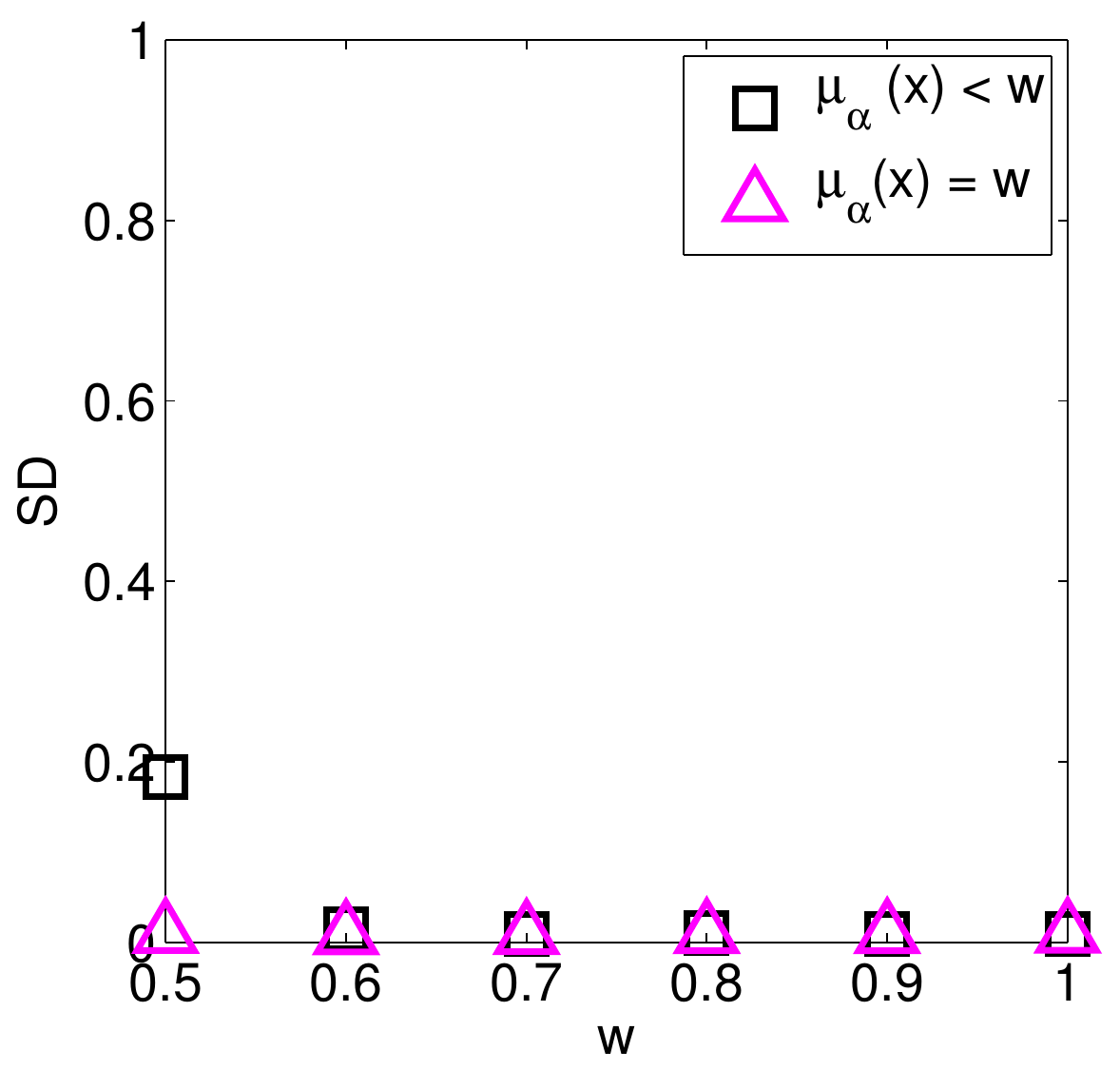}
                \caption{SD at time $t = 2000$ for different values of $w$ under the two different updating algorithms.}
                \label{fig:cmpeqAPD075}
        \end{subfigure}
        ~ 
        \begin{subfigure}[t]{0.7\columnwidth}
                \centering
                \includegraphics[width=\textwidth]{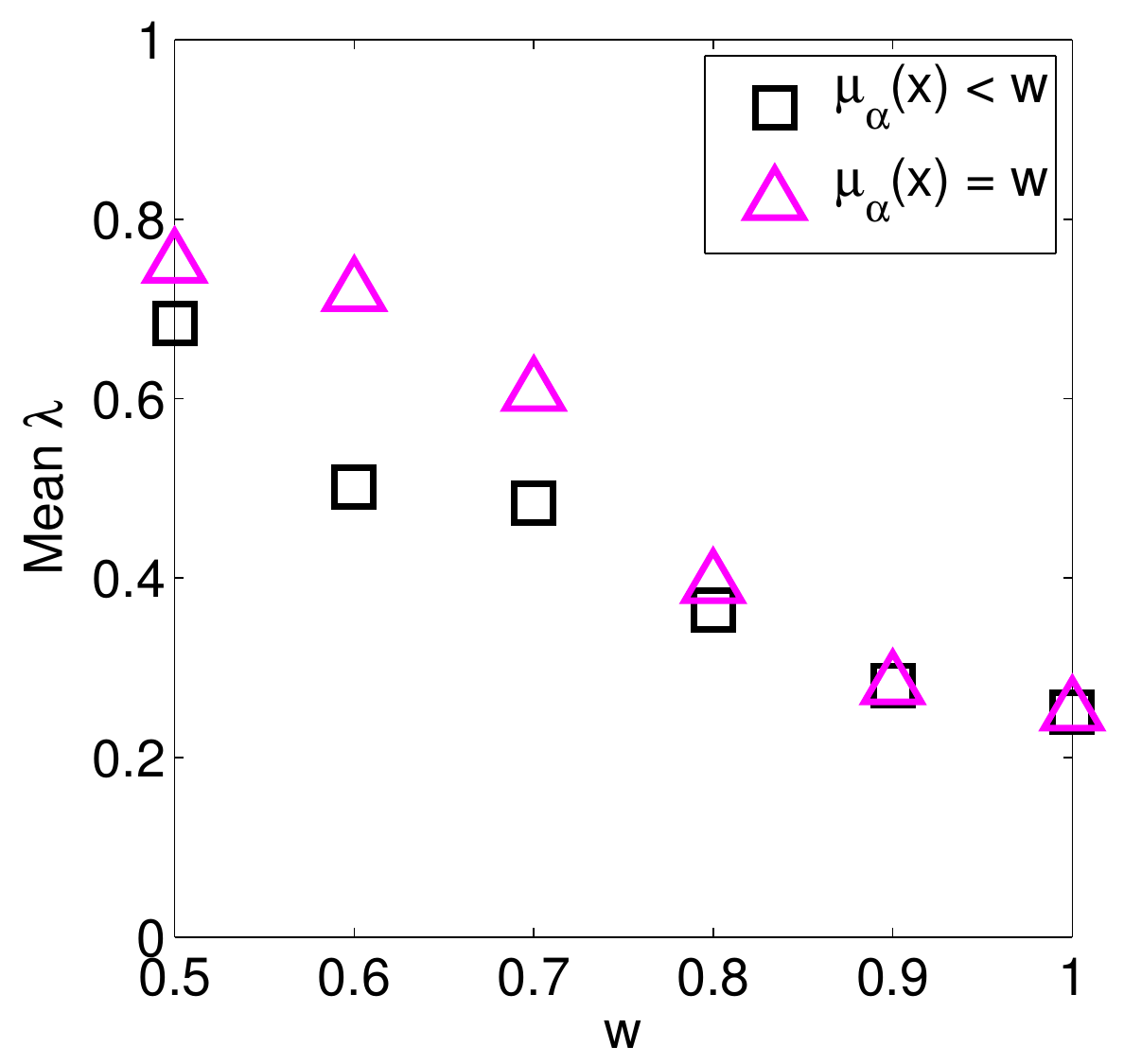}
                \caption{Mean $\lambda$ at time $t = 2000$ for different values of $w$ under the two different updating algorithms.}
                \label{fig:cmpeqmean075}
        \end{subfigure}
	\caption{$x_1 \sim U[0.25,0.75]$, $x_2 \sim U[0,0.5]$, values of $w$ vary as indicated. The population of agents converges to different values of $\lambda$ depending on the updating algorithm used.}
	\label{fig:cmpeq075}
\end{figure}

When the updating condition requires that $\mu_{\alpha_i}(\vec{x}) \neq w$, agents converge to a shared weight $\lambda$ even when their reliability $w \leq 0.5$. However, the values of $\lambda$ to which agents converge are not the same as those converged to under the first assertion algorithm.

These results show that the combination weights that the agents converge to are dependent on firstly the distribution of elements in the conceptual space, as determined by the environment, and secondly the reliability of the agents in the space. We go on to give an analysis of these results, and to prove some results relating the values of the final weights to the distribution of objects within the conceptual space.

\subsection{Analysis}
\label{sec:analysis}
We wish to analyse the results seen in section \ref{sec:simulations} in order to predict the value to which $\lambda$  will converge and also the extent to which it will converge across agents as measured by the standard deviation of the $\lambda_i$ across agents. To do so, we start with some analysis of the particular space and labels we have used before giving some more general results.

Within the results in section \ref{sec:simulations}, agents have labels $L_1 = L_2 = <1, U[0,1]>$ to describe the conceptual space $\Omega = [0,1]^2$. In this case, we have $\mu_{L_i}(x) = x_i$. Each assertion $\alpha_i$ is made exactly when each component $\mu_{\pm L_1}(x)>0.5$, $\mu_{\pm L_2}(x) >0.5$ (since otherwise another $\alpha_j$ would be maximal). We can therefore split up the conceptual space into quadrants corresponding to where each $\alpha_i$ is asserted, displayed in figure \ref{fig:quadrants}. Each quadrant where $\alpha_i$ is asserted is called $R_i$.

\begin{figure}[h!]
	 \centering
	  \def\svgwidth{0.25\textwidth}
	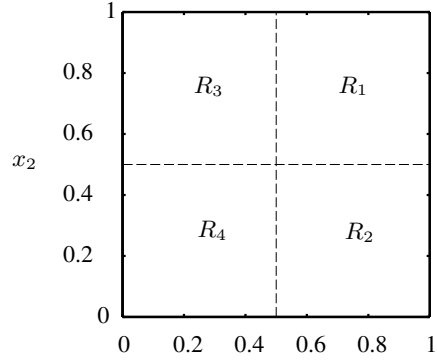
	\vspace{10pt}
	 \caption{Each assertion $\alpha_i$ is made when $\vec{x} = (x_1, x_2)$ falls in $R_i$}
	\label{fig:quadrants}
\end{figure}

To investigate the value to which the $\lambda_i$ converge, we look at the quantity $(A - \lambda)$. If $(A - \lambda)$ is positive, $\lambda$ will increase, and if it is negative, $\lambda$ will decrease. We therefore look at the circumstances that lead to $\lambda$ increasing or decreasing. We do this on a case by case basis depending on which assertion is being made.

\subsubsection*{Case: $\alpha_1$ is asserted}
When $\alpha_1$ is asserted, $\mu_{L_1}(x_1) = x_1 > 0.5$, $\mu_{L_2}(x_2) = x_2 > 0.5$. We wish to know when $A > \lambda$

\begin{align*}
\label{eq:gt}
A &= \frac{w - \mu_{L_2}(x_2)}{\mu_{L_1}(x_1) - \mu_{L_2}(x_2)} = \frac{w - x_2}{x_1 - x_2} \geq \lambda
\end{align*}

$\lambda$ is updated when $\mu_{\alpha_1} = \lambda x_1 + (1 - \lambda) x_2 < w$, i.e. when

\[
w - x_2 \geq \lambda(x_1 - x_2)
\]

Then if $(x_1 - x_2) > 0$ then $A \geq \lambda$, and the update is a positive increment. Otherwise the reverse holds, i.e. $A \leq \lambda$ and the update is a negative increment. 

\subsubsection*{Case: $\alpha_2$ is asserted}
When $\alpha_2$ is asserted, $\lambda$ is updated when $\mu_{\alpha_2} = \lambda x_1 + (1 - \lambda) (1 - x_2) < w$, i.e. when

\[
w -(1 -  x_2) \geq \lambda(x_1 + x_2 - 1)
\]

Then if $(x_1 + x_2 - 1) > 0$, $A \geq \lambda$, and the update is a positive increment. Otherwise the reverse holds, i.e. $A \leq \lambda$ and the update is a negative increment. 

\subsubsection*{Case: $\alpha_3$ is asserted}
When $\alpha_3$ is asserted, $\lambda$ is updated when $\mu_{\alpha_3}(\vec{x}) = \lambda(1 -  x_1) + (1 - \lambda) x_2 < w$, i.e. when

\[
w - x_2 \geq \lambda(1 - x_1 - x_2)
\]

Then if $(1 - x_1 - x_2) > 0$, $A \geq \lambda$, and the update is a positive increment. Otherwise the reverse holds, i.e. $A \leq \lambda$ and the update is a negative increment. 

\subsubsection*{Case: $\alpha_4$ is asserted}
When $\alpha_4$ is asserted, $\lambda$ is updated when $\mu_{\alpha_4}(\vec{x}) = \lambda(1 -  x_1) + (1 - \lambda) (1 - x_2) < w$, i.e. when

\[
w - (1 - x_2) \geq \lambda(x_2 - x_1)
\]

Then if $(x_2 - x_1) > 0$, $A \geq \lambda$, and the update is a positive increment. Otherwise the reverse holds, i.e. $A \leq \lambda$ and the update is a negative increment.  

Each of these cases can be represented graphically, since the conditions are determined by the lines $x_1 = x_2$ and $x_1 = 1 - x_2$. This is illustrated in figure \ref{fig:octants}. We call areas of the space where a positive update is made \emph{positive regions} and areas of the space where negative updates are made \emph{negative regions}.

\begin{figure}
	 \centering
	  \def\svgwidth{0.25\textwidth}
	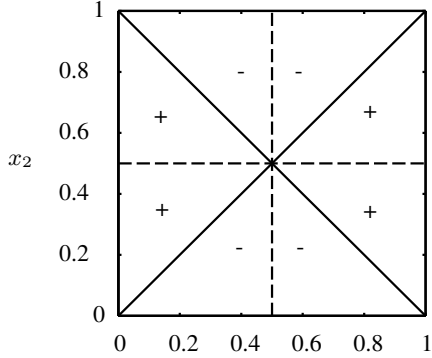
	\vspace{10pt}
	  \caption{The type of update made when $\vec{x} = (x_1, x_2)$ falls in each area of the space}
	\label{fig:octants}
\end{figure}

We can now prove some results concerning the final value of $\lambda$ across the population of agents.

\begin{thm}
Suppose that agents have labels $L_1 = L_2 = <1, U(0,1)>$ and all agents have weight 1. Agents update their concepts according to the language game described using updating model 1. Suppose that there is a probability $p^+$ of $\vec{x}$ falling in a positive region and probability $p^- = 1 - p^+$ of $\vec{x}$ falling in a negative region. Then the expected value of $\lambda$ converges to $p^+$
\end{thm}

\begin{proof}
Whenever $\vec{x}$ falls in the positive region, $A \geq 1$ and hence we set $A = 1$. For example, suppose $\vec{x} \in R_1$. Then 
 \[
 A = \frac{1 - x_2}{x_1 - x_2}
 \]
Now $1 - x_2 \geq 0$ so the sign of $A$ is determined by $x_1-x_2$. If  $\vec{x}$ falls in the positive quadrant then $x_1 > x_2$ so
\[
A = \frac{1 - x_2}{x_1 - x_2} \geq 1
\]
So $A = 1$ and the update is $\lambda_{n+1} = \lambda_n + h(1 - \lambda_n)$ 
 
If  $\vec{x}$ falls in the negative quadrant then $x_1 < x_2$ so
 \[
 A = \frac{1 - x_2}{x_1 - x_2} \leq 0
 \]
So $A = 0$ and the update is $\lambda_{n+1} = \lambda_n - h\lambda_n$. Similar arguments can be made for the other quadrants, giving us the following updating rule:
\[
\lambda_{t+1} = 
\begin{cases}
	\lambda_t(1-h) &\text{if $\vec{x} \in R^-$ } \\
  	\lambda_t(1-h)  + h &\text{if $\vec{x} \in R^+$}
\end{cases}•
\]
Then consider the behaviour of the expected value of $\lambda$ over time:
\begin{align*}
E(\lambda_{t+1}) & = E(\lambda_{t+1}| \vec{x} \in R^+) p^+ + E(\lambda_{t+1}| \vec{x} \in R^-) p^-\\
& = E(\lambda_t(1 - h) + h|  \vec{x} \in R^+)p^+ \\
&\qquad+ E(\lambda_t(1 - h)|  \vec{x} \in R^-)p^- \\
& = (1-h)E(\lambda_t) + p^+h
\end{align*}
As $t\rightarrow \infty$, we have
\begin{align*}
E(\lambda) &= E(\lambda)(1-h) + p^+h\\
&=p+
\end{align*}
\end{proof}

We have therefore related the value of $\lambda$ to which agents converge to the distribution of elements $\vec{x}$ in the conceptual space $\Omega$.

\begin{thm}
\label{thm:E}
Suppose agents are equipped with labels $L_1$, $L_2$ and combine and update these label according to the language game described. Then the expected value of $\lambda$, $E(\lambda) = E(A)$.
\end{thm}

\begin{proof}
To obtain $E(\lambda)$, consider:
\[
\lambda_{t+1} = \lambda_t(1 - h) + hA
\]
\begin{align*}
E(\lambda_{t+1}) & = E(\lambda_t(1 - h) + hA)\\
\label{eq:exp}
& = (1-h)E(\lambda_t) + hE(A)
\end{align*}
Then as $t\rightarrow \infty$, we have
\begin{align*}
E(\lambda) &= E(\lambda)(1-h) + hE(A)\\
&=E(A)
\end{align*}
\end{proof}

This result shows that the weighting given to the compound concepts $\alpha_i$ may be directly predicted from the distribution of elements in the conceptual space and the membership functions used to categorise the constituent labels $L_j$.  

If we consider updating model 2, where agents update whenever $\mu_{\alpha_i}(\vec{x}) \neq w$, we may also obtain an expression for the variance of the $\lambda_i$ across agents.

\begin{thm}
Suppose agents have labels $L_1$, $L_2$ and that agents update according to updating model 2: $\mu_{\alpha_i}(\vec{x}) \neq w$. Then the variance of the $\lambda$ across agents is $Var(\lambda) = \frac{h}{2-h}Var(A)$	
\end{thm}
\begin{proof}
We obtain $Var(\lambda)$ by firstly calculating $E(\lambda^2)$:
\begin{align*}
E(\lambda_{t+1}^2) & = E((\lambda_t(1 - h) + hA)^2)\\
&=E((1-h)^2\lambda_t^2 +2h(1-h)\lambda_tA +h^2A^2))\\
& = (1-h)^2E(\lambda_t^2) +2h(1-h)E(\lambda_t)E(A) + h^2E(A^2)\\
\end{align*}
We  may then calculate
\begin {align*}
Var(\lambda_{t+1}) &= E(\lambda_{t+1}^2) - (E(\lambda_{t+1}))^2 \\
&= (1-h)^2E(\lambda_t^2) +2h(1-h)E(\lambda_t)E(A)\\
&\qquad + h^2E(A^2) - (E(\lambda)(1-h) + hE(A))^2\\
&=  (1-h)^2E(\lambda_t^2) + 2h(1-h)E(\lambda_t)E(A)\\
& \qquad + h^2E(A^2) - (1-h)^2(E(\lambda_t))^2\\
& \qquad - 2h(1-h)E(\lambda_t)E(A) - h^2 (E(A))^2\\
&=(1-h)^2Var(\lambda_t) - h^2Var(A)
\end{align*}
As $t \rightarrow \infty$, we have
\begin{align*}
Var(\lambda) &= (1-h)^2Var(\lambda) - h^2Var(A)\\
&= \frac{h}{2 - h} Var(A)
\end{align*}
\end{proof}
We can further examine how quickly the mean and variance of $\lambda_i$ approach that of $A$. We obtain an expression for $E(\lambda_t)$ and $Var(\lambda_t)$ in terms of t and solve. The speed at which the $\lambda_i$ converge to the resting state is dependent on $h$. 

\begin{thm}
Suppose agents are equipped with labels $L_1$, $L_2$ and combine and update these label according to the language game described, using updating model 2. Then the number of timesteps until $ |E(\lambda_t) - E(A)| \leq \epsilon$ for some small $\epsilon$ is $t \geq (\log(\epsilon) - \log(|E(\lambda_0) - E(A)|))/\log(1-h)$, and the number of timesteps until $|Var(\lambda_t) - \frac{h}{2-h}Var(A)| \leq \epsilon$ is 
\[t \geq \frac{(\log(\epsilon) - \log(|Var(\lambda_0) - \frac{h}{2-h}Var(A)|))}{2\log(1-h)}\]
\end{thm}
\begin {proof}

We firstly obtain an expression for $E(\lambda_t)$ in terms of $t$, $h$, $E(A)$ and $E(\lambda_0)$

\begin{align*}
E(\lambda_{t}) & = E(\lambda_{t-1})(1 - h) + hE(A)\\
& = E(\lambda_0)(1-h)^t + E(A)\sum_{k = 1}^t(h(1-h)^{t-1})\\
&= E(\lambda_0)(1-h)^t + E(A)(1 - (1-h)^t)\\
\end{align*}

Now, consider

\begin{align*}
\epsilon &\geq |E(\lambda_t) - E(A)|\\
& = |E(\lambda_0) -E(A)|(1-h)^t\\
t &\geq (\log(\epsilon) - \log(|E(\lambda_0) - E(A)|))/\log(1-h)
\end{align*}

To calculate the number of timesteps needed until $Var(\lambda)$ has reached its resting state, we again obtain an expression for $Var(\lambda_t)$ in terms of $t$, $h$, $Var(A)$ and $Var(\lambda_0)$.
\begin{align*}
Var{\lambda_t} &= (1-h)^2Var(\lambda_(t-1)) - h^2Var(A)\\
&= Var(\lambda_0)(1-h)^{2t} + Var(A)\sum_{k = 1}^t h^2(1-h)^{2t}\\
& = Var(\lambda_0)(1-h)^{2t} + \frac{h}{2-h}Var(A)(1 - (1-h)^{2t})
\end{align*}

Again, consider 
\begin{align*}
\epsilon &\geq |Var(\lambda_t) - \frac{h}{2-h}Var(A)|\\
& = |Var(\lambda_0) - \frac{h}{2-h}Var(A)|(1-h)^{2t}\\
t &\geq \frac{(\log(\epsilon) - \log(|Var(\lambda_0) - \frac{h}{2-h}Var(A)|))}{2\log(1-h)}
\end{align*}

\end{proof}

To illustrate these results, we ran simulations of the language game with 1000 agents each with $L_1 = L_2 = <\!\!1, U(0,1)\!\!>$, $w = 1$, $x_1 \sim U[0.25,0.75]$, $x_2 \sim U[0,0.5]$, and $h \in \{10^{-2}, 10^{-3}, 10^{-4}, 10^{-5}\}$. Figure $\ref{fig:limE}$ shows the values of $E(\lambda_t)$ over time obtained from simulations, together with the predicted value of $E(\lambda_t)$ and also the resting state $E(\lambda)$. 

\begin{figure} [htbp]
                \centering
                \includegraphics[width=\columnwidth]{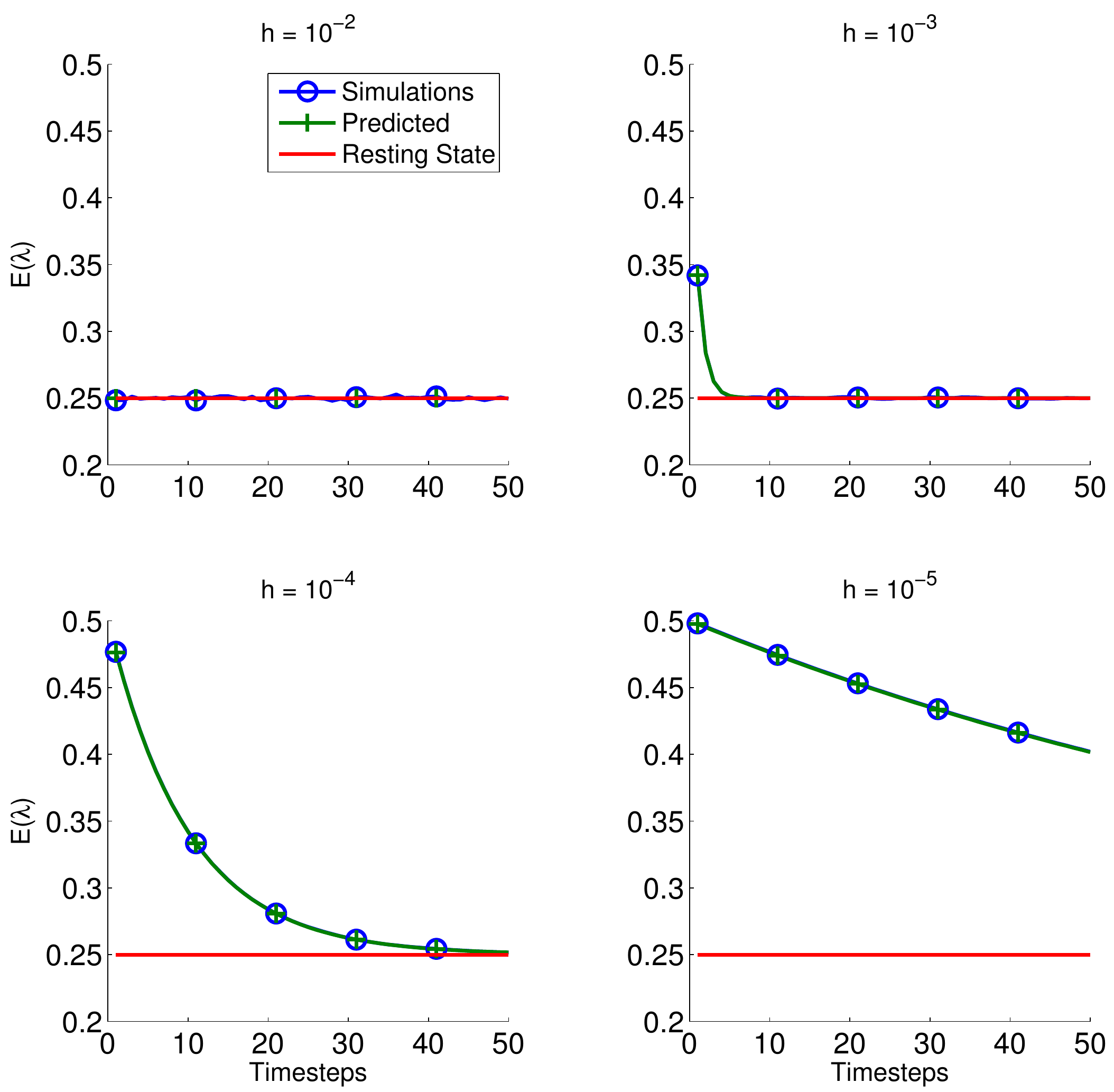}
                \caption{Predicted value of $E(\lambda)$ and actual value of $E(\lambda)$ over time for different values of $h$. The rate at which $E(\lambda)$ approaches its resting state is slower for smaller $h$.}
                \label{fig:limE}
        \end{figure}

 Figure $\ref{fig:limVar}$ shows the values of $Var(\lambda_t)$ over time obtained from simulations, together with the predicted value of $Var(\lambda_t)$ and also the resting state $Var(\lambda)$.
 
\begin{figure}[t]
                \centering
                \includegraphics[width=\columnwidth]{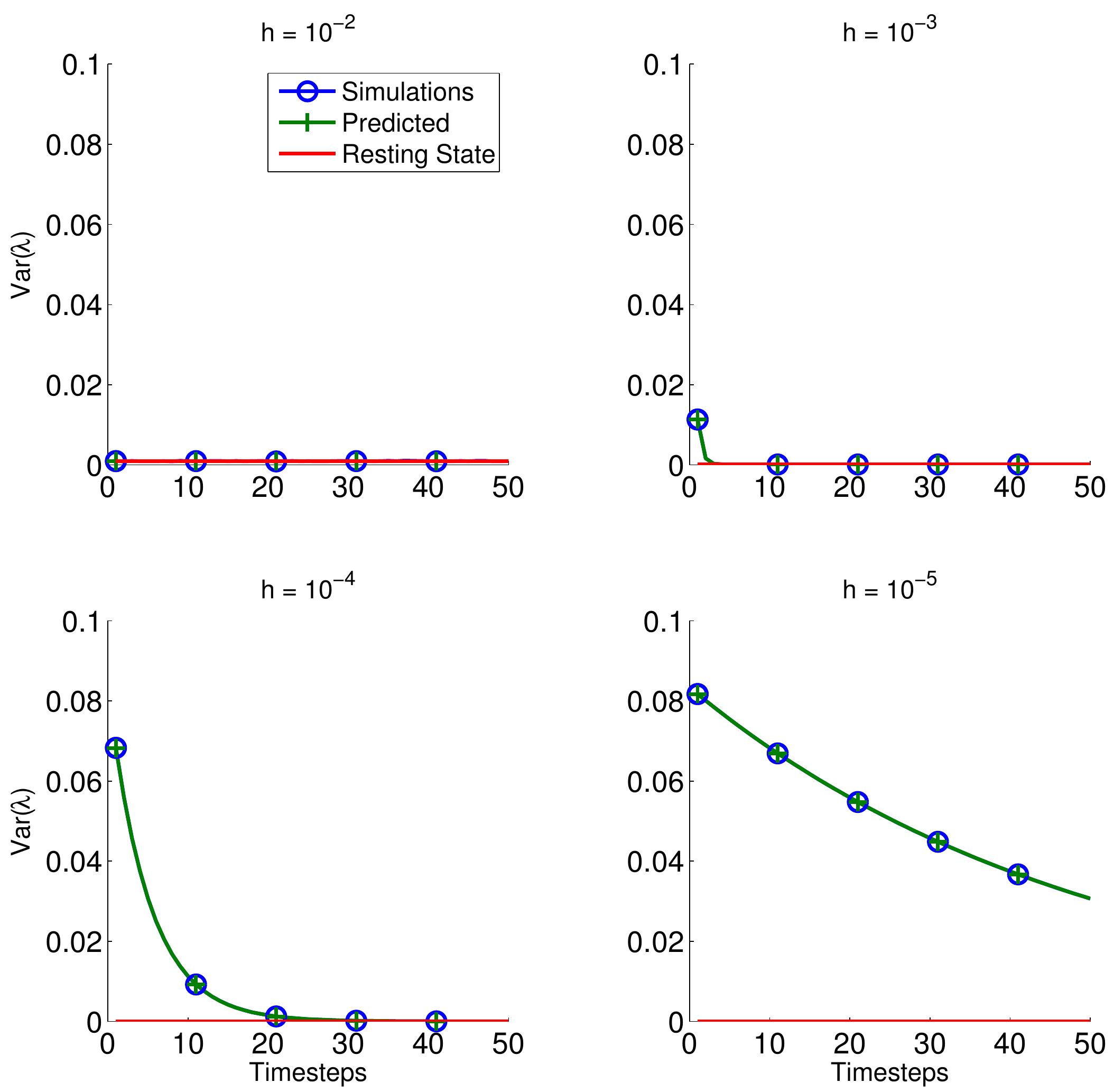}
                \caption{Predicted value of $Var(\lambda)$ and actual value of $Var(\lambda)$ over time for different values of $h$. The rate at which $Var(\lambda)$ approaches its resting state is slower for smaller $h$}
                \label{fig:limVar}
\end{figure}

These results illustrate that the rate at which agents converge to the resting state is dependent on the value $h$ by which agents adopt others agents' viewpoints. Larger values of $h$ allow faster convergence to the resting state, however that resting state will have a larger variance when $h$ is larger.

\section{DISCUSSION}
\label{sec:disc}
Characterising concepts as a weighted sum of attributes is seen throughout the literature \cite{gard2004, hamp1987, lakoffhedges, zadehhedges}. However, this has been proposed in an ad hoc fashion, and many concepts do not adhere to this formulation. Further, no mechanism for determining the weights has been proposed. We have developed a hierarchical model of concept combination from which the characterisation of concepts as a weighted sum of attributes arises naturally, summarised in section \ref{sec:bkg}. We have implemented a simple version of this model within the language game framework in which agents make assertions that consist of a weighted sum of two constituent concepts. We have shown that within a multi-agent simulation of a community of such language users agents can converge to shared weightings (section \ref{sec:simulations}). The weights $\lambda_i$ to which the community of agents converge are related to the distribution of elements $x_j$ in the conceptual space, the membership functions $\mu_{L_j}(x_j)$ used for the constituent concepts, and the reliability $w$ of the agents. We have derived explicit expressions relating the mean of the final weights $\lambda_i$ across all agents to the distribution of the $x_j$, $\mu_{L_j}(x_j)$, and $w$. In a modified updating model that is more amenable to analysis, we have derived expressions for the variance of the $\lambda_i$ in terms of $x_j$, $\mu_{L_j}(x_j)$, $w$ and also $h$, the rate at which agents adopt other agents' concepts. When $h$ is small, i.e. agents are slower to adopt others' concepts, the variance of the $\lambda_i$ is smaller, perhaps indicating that more robust concepts are formed. We have also derived expressions for the speed at which the resting distribution of the $\lambda_i$ are approached, dependent on the value of $h$ (section \ref{sec:analysis}). The results from this model indicate how weightings in a weighted sum of concepts can be related to the elements encountered in a conceptual space. 

 This model is of course extremely simple, and numerous extensions are ongoing. The use of more complex labels $L_i$ is under investigation, with labels being extended to multiple dimensions. We are also developing the updating algorithm in order to combine more than two labels and to include noise in agents' labels. In previous work \cite{ettie, melinghedges}, the agents in the simulations have had varying reliability $w$, which could be incorporated into these simulations. Future work will also incorporate the theoretical aspects alluded to in section \ref{sec:bkg} that account for non-compositional properties such as emergent attributes.

\ack
Martha Lewis gratefully acknowledges support from EPSRC Grant No. EP/E501214/1

\bibliography{../phd}

\end{document}